\documentclass[11pt]{article}

\input{settings/packages.tex}
\theoremstyle{plain}
\newtheorem{theorem}{Theorem}[section]

\newtheorem{lemma}[theorem]{Lemma}
\newtheorem{corollary}[theorem]{Corollary}
\newtheorem{fact}[theorem]{Fact}
\theoremstyle{definition}
\newtheorem{definition}[theorem]{Definition}

\theoremstyle{remark}
\newtheorem{remark}[theorem]{Remark}

\newcommand{\OPT}{\operatorname{OPT}}

\newcommand{\cC}{\mathcal{C}}
\newcommand{\R}{\mathbb{R}}
\newcommand{\conv}{\mathrm{conv}}

\newcommand{\E}{\operatorname{\mathbb{E}}}
\newcommand{\PP}{\operatorname{\mathbb{P}}}

\newcommand{\prob}[1]{\PP \left[ #1 \right]}

\newcommand{\1}{\mathds{1}}

\newcommand{\cA}{\mathcal{A}}

\newcommand{\cI}{\mathcal{I}}
\newcommand{\cJ}{\mathcal{J}}

\newcommand{\cF}{\mathcal{F}}
\newcommand{\cG}{\mathcal{G}}

\DeclarePairedDelimiter{\floor}{\lfloor}{\rfloor}
\DeclarePairedDelimiter{\ceil}{\lceil}{\rceil}

\DeclareMathOperator*{\argmin}{argmin}

\newcommand{\firstpass}{\textsc{Fair-Reservoir}\xspace}
\newcommand{\secondpass}{\textsc{Randomized-Fair-Streaming}\xspace}

\newcommand{\jtodo}[1]{\todo[color=blue!30, inline]{\emph{Jakub}: #1}}
\newcommand{\mtodo}[1]{\todo[color=blue!30, inline]{\emph{Marwa}: #1}}

\graphicspath{{./figs/}}

\usepackage[margin=1in]{geometry}
\usepackage[parfill]{parskip}
 
 \usepackage[round]{natbib}
\renewcommand{\cite}[1]{\citep{#1}}
\renewcommand{\bibname}{References}

\usepackage{authblk}

\begin{document}

\title{Fairness in Submodular Maximization over a Matroid Constraint}

\author[1]{Marwa El Halabi}
\author[2]{Jakub Tarnawski}
\author[3]{Ashkan Norouzi-Fard}
\author[4]{Thuy-Duong Vuong}

\affil[1]{Samsung - SAIT AI Lab, Montreal}
\affil[2]{Microsoft Research}
\affil[3]{Google Zurich}
\affil[4]{Stanford University}

\date{}
\maketitle

\begin{abstract}
Submodular maximization over a matroid constraint is a fundamental problem with various applications in machine learning. Some of these applications involve decision-making over datapoints with sensitive attributes such as gender or race. In such settings, it is crucial to guarantee that the selected solution is fairly distributed with respect to this attribute.
Recently, fairness has been investigated in submodular maximization under a cardinality constraint in both the streaming and offline settings, however the more general problem with matroid constraint has only been considered in the streaming setting and only for monotone objectives. This work fills this gap. We propose various algorithms and impossibility results offering different trade-offs between quality, fairness, and generality.
\end{abstract}
\section{Introduction}


Machine learning algorithms are increasingly used in decision-making processes.
This can potentially lead to the introduction or perpetuation of bias and discrimination in automated decisions. Of particular concern are sensitive areas such as education, hiring, credit access, bail decisions, and law enforcement~\cite{executive2016,blueprint22,reportEU22}.
There has been a growing body of work
attempting to mitigate these risks
by
developing \emph{fair} algorithms
for fundamental problems
including
classification \cite{ZafarVGG17},  ranking \cite{CelisSV18}, clustering \cite{Chierichetti0LV17}, voting \cite{CelisHV18}, matching \cite{Chierichetti0LV19}, influence maximization \cite{TsangWRTZ19}, data summarization \cite{CelisKS0KV18}, and many others. 

In this work, we address fairness in the fundamental problem of submodular maximization over a matroid constraint, in the offline setting.
Submodular functions model a diminishing returns property that naturally occurs in a variety of machine learning problems
such as active learning \cite{GolovinK11}, data summarization \cite{LinB11}, feature selection \cite{DasK11}, and recommender systems \cite{El-AriniG11}.
Matroids represent a popular and expressive notion of independence systems that encompasses a broad spectrum of useful constraints, e.g.~cardinality, block cardinality, linear independence, and connectivity constraints.

Several definitions of algorithmic fairness have been proposed in the literature,
but no universal metric exists.
Here we adopt the common notion used in many prior studies \cite{CelisHV18, CelisKS0KV18, CelisSV18,Chierichetti0LV17,Chierichetti0LV19}
that requires a solution to be \emph{fair} with respect to a sensitive attribute such as race or gender.
Formally, given a set $V$ of $n$ items, each item is associated with a color $c$ representing a sensitive attribute. Let $V_1, \ldots, V_C$ denote the corresponding $C$ \emph{disjoint} groups of items of the same color.
A set $S \subseteq V$ is \emph{fair} if $\ell_c \leq |S \cap V_c| \leq u_c$ for all $c$, for some chosen lower and upper bounds $\ell_c, u_c \in \mathbb N$.
This notion subsumes several other fairness definitions, e.g.~diversity rules \cite{cohoon2013, biddle2006adverse}, statistical parity \cite{dwork2012fairness}, 
or proportional representation rules \cite{monroe1995fully, Brill2017}.
See \citet[Section 4]{CelisHV18} for a more detailed overview.

Without fairness, the problem of maximizing a submodular function over a matroid constraint in the offline setting has been studied extensively. For monotone objectives, a tight $(1-1/e)$
-approximation is known  \cite{calinescu2011maximizing,Feige98}, while for non-monotone objectives the best known approximation is $0.401$ given in \cite{Buchbinder2023}.  An information-theoretic hardness of $0.478$ was also shown for the non-monotone case in \cite{GharanV11}.

Fair submodular maximization has been considered under a \emph{cardinality} constraint,  in both the offline and streaming models. 
\citet{CelisHV18} presented a tight $(1-1/e)$-approximation to the problem in the offline setting with monotone objectives. 
In the streaming setting,
\citet{HalabiMNTT20} gave a one-pass $0.3178$-approximation algorithm in the monotone case, and a one-pass $0.1921 q$-approximation algorithm in the non-monotone case, where $q = 1 - \max_c \tfrac{\ell_c}{|V_c|}$.\footnote{
These approximations are better than the ones stated in \cite{HalabiMNTT20}. They follow from their results by plugging in the state-of-the-art approximations of \citet[Theorem 1 and 5]{Feldman2021a} for streaming submodular maximization over a matroid constraint.} 
They also showed that achieving a better than $q$-approximation in this case requires $\Omega(n)$ memory.

Fair submodular maximization under a general \emph{matroid} constraint was only studied in the streaming setting, and only for monotone objectives. \citet{el2023fairness} provided a one-pass $1/2$-approximation  algorithm that uses time and memory exponential in the rank of the matroid; this factor is tight  \cite{FeldmanNSZ20}.
They also gave a polynomial-time two-pass $1/11.656$-approximation
which uses $O(k \cdot C)$ memory and violates the fairness lower bounds by a factor~$2$. 

While some applications necessitate algorithms that use small memory compared to the data size, others can afford a larger memory budget. 
In this work, we focus on the classic offline setting, which has not yet been studied despite being more basic than streaming.
Our results show that it is possible to obtain stronger performance guarantees in this case. 
We also consider both monotone and non-monotone objectives, which cover a wider range of applications.

\subsection{Our contributions}
We present several algorithms and impossibility results for fair matroid submodular maximization (FMSM) in the offline setting, with trade-offs between quality, fairness, and generality. 
\Cref{results-summary} summarizes our approximation results and prior ones for this problem.

\vspace{-1.2pt}
\begin{table*}
    \centering
     \caption{Summary of results for FMSM in the offline setting (only  polynomial-time algorithms with respect to $n, k$ and $C$ are included). In the ``Fairness Approx." column, a $(\theta_1, \theta_2)$ entry means the algorithm's output satisfies $|S \cap V_c| \in [\theta_1 \ell_c, \theta_2 u_c]$. The parameters $\epsilon>0$, $\beta \in [0,1/2]$ are arbitrary constants, and $r =  \min_{x \in  P_\cF  \cup (\1 - P_\cF)} \| x \|_\infty$.}
    \label{results-summary}
    \begin{adjustbox}{width=1\textwidth}
    \begin{tabular}{cccc}
    \toprule
       \bf Function  &     \bf Matroid &     \bf Approx. Ratio &      \bf Fairness Approx.  \\ \midrule
       \midrule
        Monotone  & General &  $1 - 1/e$ (Thm. \ref{approx-violated-fairness}) & $(1,1)$ in expectation \\
       \cmidrule{1-3}
       Non-monotone  & General & $(1 - r - \epsilon)/4$ (Thm. \ref{approx-violated-fairness}) & $(1 - \sqrt{3{\ln(2 C)}/{\ell_c}}, 1 + \sqrt{3{\ln(2C)}/{u_c}})$  \\   \cmidrule{1-4}
        Monotone &  General & $1/(4+\epsilon)$ \cite{el2023fairness} & $(1/2, 1)$  \\    \cmidrule{1-4}
        Non-monotone &  General & $(1 - \beta)/(8 + \epsilon)$ (Thm. \ref{approx-constant-violation}) &  $(\beta, 1)$  \\     \cmidrule{1-4}
        Monotone & Uniform & $1 - 1/e$ \citep{CelisHV18} & $(1,1)$  \\     \cmidrule{1-4} 
        Non-monotone & Uniform & $0.401 ( 1 - r)$ (Thm. \ref{approx-unif}) & $(1,1)$  \\     \cmidrule{1-4}
        Monotone decomposable & General & $1 - 1/e$ (Thm. \ref{approx-decomposable}) & $(1,1)$  \\     \cmidrule{1-4}
        Non-monotone decomposable & General & $(1 - r - \epsilon)/4$ (Thm. \ref{approx-decomposable}) & $(1,1)$  \\
        \bottomrule
    \end{tabular}
    \end{adjustbox}
\end{table*}

First, we observe that the streaming algorithms of \citet{el2023fairness} for monotone objectives apply in the offline setting. This yields a $1/2$-approximation in \emph{exponential time} with respect to the rank of the matroid. And, with a slight modification, the two-pass algorithm therein achieves a $1/(4+\epsilon)$-approximation in polynomial time, with a factor-$2$ violation of the fairness lower bounds.
We extend the latter result to non-monotone objectives, obtaining a $(1 - \beta)/(8 + \epsilon)$-approximation with a $1/\beta$ violation of the fairness lower bounds, for any $\beta \in [0,1/2]$ (\cref{approx-constant-violation}). 

It is of course preferable to obtain algorithms that satisfy the fairness constraints exactly.
However, we give strong evidence that this is very challenging in general.
A commonly used approach for submodular maximization problems is \emph{relax-and-round},
which consists of first solving a continuous relaxation of the problem obtained using the \emph{multilinear extension} (see  definition in \cref{multilinear-ext}) of the objective over the convex hull $P_\cF$ of the domain, then rounding the fractional solution.
Most 
studied submodular maximization problems
that admit a constant-factor approximation algorithm
also admit one based on the relax-and-round approach.
Unfortunately,
we show that for FMSM,
this approach cannot yield a better than $O(1/\sqrt{n})$-approximation in general (\cref{integrality-gap}),
even for monotone objectives.
 
Yet, 
if we allow the fairness constraint to be satisfied only \emph{in expectation}, we can adapt  this approach to obtain a tight $(1 - 1/e)$-approximation in expectation for monotone objectives. For non-monotone objectives, the approximation becomes $(1 -  r - \epsilon)/4$ where $r$ is the minimum $\ell_\infty$-norm of any vector in $P_\cF$ or its complement. Both algorithms are also guaranteed to only violate the fairness constraints by a multiplicative factor expected to be small in practice (\cref{approx-violated-fairness}).  

In fact,
the factor $r$ is closely tied to the hardness of non-monotone FMSM.
We show that no algorithm can achieve a better than $(1 - r)$-approximation in sub-exponential time, even without the matroid constraint (\cref{non-monotone-hardness}). In the case of uniform or no matroid, $1-r$ coincides with the notion of excess ratio $q= 1 - \max_c \tfrac{\ell_c}{|V_c|}$ studied by \citet{HalabiMNTT20}. As such, this inapproximability result extends the one provided therein for the streaming setting.

While the $(1-r)$ factor in the approximation for non-monotone FMSM cannot be improved even for the uniform matroid, we show that the $1/4$ factor can be improved to $0.401$ in this case (\cref{approx-unif}). This improves over the $0.1921 q$-approximation of \cite{HalabiMNTT20} in the streaming setting.

Finally, we study an interesting subclass of FMSM, which we call decomposable FMSM, where the objective decomposes into submodular functions on the equivalence classes of the matroid (see \cref{equiv-classes}) or on the color groups of the fairness constraint. This subclass enables us to consider  fairness in the important \emph{submodular welfare problem} (see \cref{SWP}).  We provide a tight $(1 - 1/e)$-approximation to decomposable FMSM for monotone objectives and a $(1 -  r - \epsilon)/4$ for non-monotone objectives, in expectation, \emph{without} violating the fairness constraint (\cref{approx-decomposable}).


\subsection{Additional related work}

Before this paper, non-monotone FMSM was only studied in the offline setting in the special case of cardinality constraint
by \citet{yuan2023group}, who obtained a $0.2005$-approximation for a specific setting of fairness bounds where $\tfrac{\ell_c}{|V_c|} = a$ and $\tfrac{u_c}{|V_c|} = b$ for all $c$, for some constants $a, b \in [0,1]$.\footnote{\label{remark-Yuan} The approximation given in \cite{yuan2023group} is $0.401/2$ for $a \leq 1/2$ and $0.401/3$ otherwise. However, their analysis can be easily modified to show $0.401/2 = 0.2005$-approximation for any $a$.} Our results recover their guarantee (see \cref{approx-unif} and the discussion below it).

Several works have studied other special cases and variants of FMSM. 
\citet{wang2021fair} studied a special case of monotone FMSM in the streaming setting, where fairness lower and upper bounds are equal for each color, i.e., $\ell_c = u_c$ for all $c$ and without any matroid constraint. 
\citet{tang2023achieving} and \citet{tang2023beyond} studied a randomized variant of FMSM with cardinality constraint in the offline setting, where fairness constraint only needs to be satisfied in expectation.  

A closely related problem to FMSM is submodular maximization over two matroid constraints; FMSM reduces to this problem when $\ell_c = 0$ for all $c$.
\citet{Lee2010} gave a $1/(2+\epsilon)$-approximation for this problem for monotone objectives, and a $1/(4 + \epsilon)$-approximation for non-monotone objectives.

An important special case of decomposable FMSM arises by considering fairness in the {submodular welfare problem}. Without fairness, this is a well-studied problem for which a tight $(1-1/e)$-approximation is known \cite{Vondrak2008, Khot2005}. The submodular welfare problem  
 has many practical applications, including resource allocation in wireless networks, budget allocation in advertising campaigns, and  division of resources among multiple stakeholders. 
In several of these applications, it is important to ensure that the selected allocation is fair. Fairness in the submodular welfare problem has been studied by \citet{Benabbou2021}, \citet{Aziz2023}, and \citet{Sun2023}.
However, these works consider different notions of fairness than we do: fairness is imposed on the {value} received by each individual in the selected allocation, instead of the size of the allocation assigned to each group of individuals. 

In this work, we focus on the case where the color groups are disjoint. The case where groups can overlap was considered in \cite{CelisHV18} for the special case of monotone FMSM with cardinality constraint.  They show that when elements can be assigned to 3 or more colors, even determining feasibility is NP-hard. But if fairness constraints are allowed to be violated, they give a $(1 - 1/e - o(1))$-approximation algorithm, which satisfies the fairness constraint in \emph{expectation}. Our result for general monotone FMSM (\cref{approx-violated-fairness}) is based on the same tools used for this result.

Fairness in submodular maximization has also been studied under a different notion of fairness \cite{Wang2024, TsangWRTZ19, tang2023beyond} which imposes a lower bound on the value that each group derives from the solution, where groups are not necessarily subsets of $V$, and the value is modeled by a monotone submodular function.  
Both the variant of fair submodular maximization considered in this line of work and our work can be formulated as variants of multi-objective submodular maximization problems \cite{Krause2008, Udwani2018,Chekuri2010}; see \cref{sec:multiobj-appendix} for details.

\section{Preliminaries}\label{sec:prelim}

    Consider a ground set $V$ with $n$ items and a non-negative set function $f:2^V \to \R_+$. 
   We say that $f$ is \emph{submodular} if 
   $f(Y \cup \{e\}) - f(Y) \ge f(X \cup \{e\}) - f(X)$
   for any two sets $Y \subseteq X \subseteq V$ and any element $e\in V \setminus X$. Moreover, $f$ is  \emph{monotone} if $f(Y) \leq f(X)$ for any two sets $Y \subseteq X \subseteq V$,
   and \emph{non-monotone} otherwise. Throughout the paper, we assume that $f$ is given as a value oracle that computes $f(S)$ given $S \subseteq V$. 
    \paragraph{Matroids.}\label{matroids}
    A \emph{matroid} is a family of sets $\cI \subseteq 2^V$ 
    with the following properties:
    \begin{itemize}
        \item 
            if $X \subseteq Y$ and $Y \in \cI$, then $X \in \cI$; 
        \item    
            if $X, Y \in \cI$ with $|X| < |Y|$, then there exists $e \in Y$ such that $X + e \in \cI$.
    \end{itemize}
    We use $X + e$ for $X \cup \{e\}$. 
    We assume that the matroid is given as an independence oracle. We call a set $X \subseteq V$ \emph{independent} if $X \in \cI$ and a {\em base} if it is also maximal with respect to inclusion. All the bases of a matroid share the same cardinality~$k$, referred to as the {\em rank} of the matroid.
    %
    \emph{Partition matroids} are an important class of matroids where $V$ is partitioned into $\bigcup_i G_i$, where each block $G_i$ has an upper bound $k_i \in \mathbb{N}$, and a set $X$ is independent if $|X \cap G_i| \leq k_i$ for all $i$.
    
\looseness=-1 \paragraph{Fair Matroid Submodular Maximization (FMSM).}
  Let $V$ be partitioned into $C$ sets, where $V_c$ denotes items of color $c$ for all $c \in [C] = \{1,\cdots,C\}$. Given fairness bounds $(\ell_c,u_c)_{c \in [C]}$ and a matroid $\cI \subseteq 2^V$ of rank $k$, the collections of \emph{fair} sets $\cC$ and \emph{feasible} sets $\cF$ are defined as:
  \vspace{-.5pt}
   \begin{align*}
        \cC &= \{S \subseteq V \mid  \ell_c \leq |S \cap V_c| \leq  u_c \; \ \forall c \in [C]\}\,, \\
        \cF &= \cI \cap \cC \,.
   \end{align*}
    In FMSM, the goal is to find a set $S \in \cF$ that maximizes $f(S)$. We use $\OPT$ to refer to the optimal value, i.e., $\OPT = \max_{S \in \cF} f(S)$. We assume a feasible solution exists, i.e., $\cF \neq \emptyset$. An algorithm is an $\alpha$-approximation to FMSM  if it outputs a set $S \in \cF$ such that $ f(S) \ge \alpha \cdot \OPT$.

We denote the convex hull of indicator vectors of fair sets by $P_\cC = \conv(\{ \1_S \mid S \in \cC\})$ (fairness polytope),
of independent sets by $P_\cI = \conv(\{ \1_S \mid S \in \cI \})$ (matroid polytope),
and of feasible sets by $P_\cF = \conv(\{ \1_S \mid S \in \cF\})$ (feasible polytope). We also define the complement of $P_\cF$ as $\1 - P_\cF = \{x \in [0,1]^n \mid \1 - x \in P_\cF \}$,
where $\1$ is the all-ones vector.
We recall in the following lemma some useful properties shown by \citet[Lemma C.1, Corollary C.4, Theorem C.5]{el2023fairness}.
\begin{lemma}\label{polytopes-properties}
The following hold:
\begin{enumerate}[label=\alph*., ref=\alph*] 
    \item $P_\cC = \{x \in [0,1]^n : \sum_{i \in V_c} x_i \in [\ell_c, u_c] \;  \forall c \in [C]\}$.
    \item $P_\cF$ corresponds to the intersection of the matroid and fairness polytopes: $P_\cF = P_\cI \cap P_\cC$.
    \item \label{linear-opt} $P_\cF$ is \emph{solvable}, i.e., linear functions can be maximized over it in polynomial time. 
\end{enumerate}
\end{lemma}

    The following lemma from \citep[Lemma 2.2]{Buchbinder2014} will be useful in the non-monotone case.
    \begin{lemma}\label{bounded-sampling}
    Let $g$ be a non-negative submodular function and let $B$ be a random subset of $V$ containing every element of $V$ with probability at most $p$ (not necessarily independently). Then $\E[g(B)] \geq (1 - p)g(\emptyset)$.
    \end{lemma}
    
    \paragraph{Multilinear extension.} \label{multilinear-ext} An important concept in submodular maximization is the multilinear extension $F:[0,1]^n \to \R_+$ of a submodular function $f$:
    \[
    F(x) = \E[f(R(x))] = \sum_{S \subseteq V} f(S) \prod_{i \in S} x_i \prod_{j \in V \setminus S} (1 - x_j), 
    \]
    where $R(x)$ is the set obtained by independently selecting each element $i \in V$ with probability $x_i$.

    \paragraph{Decomposable FMSM.}
    We also study a subclass of the FMSM problem where the objective decomposes into submodular functions on the equivalence classes of the matroid or on the color groups of the fairness constraint.
    We recall the notion of equivalent elements in a matroid \citep[Definition I.3]{Chekuri2010}.

\begin{definition}\label{equiv-classes}
Two elements $i, j \in V$ are equivalent in a matroid $\cI$ if for any set $S \in \cI$ not containing $i$ and $j$, $S + i \in \cI$ if and only if $S + j \in \cI$.
\end{definition}
This defines an equivalence relation. 
For example, the equivalence classes of a partition matroid are the partition groups $G_i$. 

\begin{definition}\label{decomposable-fcts}
Let $\cG \subseteq 2^V$ be the equivalence classes of the matroid $\cI$. We say that $f$ is a decomposable submodular function over $\cF$ if $$f(S) = \sum_{G \in \cG} f_{1, G}(S \cap G) + \sum_{c = 1}^C f_{2, c}(S \cap V_c),$$ where $f_{1, G}, f_{2, c}$ are submodular functions.
\end{definition}
\paragraph{Submodular welfare problem.} \label{SWP}
One noteworthy example of decomposable submodular functions arises from the submodular welfare problem.
In this problem, we are given a set of items $B$ and a set  of agents $A$ each with a monotone submodular utility function $w_i: 2^B \to \R_+$. The goal is to distribute the items among the agents to maximize the social welfare $\sum_{i \in A} w_i(S_i)$, where $(S_i)_{i \in A}$ are the disjoint sets of items assigned to each agent.  It is known that this problem can be written as a monotone submodular maximization over a partition matroid constraint by defining the ground set as $V = A \times B$, the objective as $f(S) = \sum_{i \in A} w_i(S_i)$ where $S_i = \{e \in B \mid (i, e) \in S \}$,
and using the partition matroid $\cI = \{S \subseteq V \mid (\forall e \in B) \ |S \cap (A \times \{e\})| \leq 1\}$ \cite{Lehmann2006}. 
In certain applications, it is important to ensure that each group of agents receives a \emph{fair} allocation of items. To that end, we introduce the \emph{fair submodular welfare problem}, where each agent is assigned exactly one color $c \in [C]$; $A_c$ is the set of agents of color $c$. Let $V_c = A_c \times B$ and $f_c(S) = \sum_{i \in A_c} w_i(S_i)$, then $f_c$ is monotone submodular and $f(S) = \sum_{c = 1}^C f_c(S \cap V_c)$ is a monotone decomposable submodular function over $\cF$. The fair submodular welfare problem is then a special case of monotone decomposable FMSM.

\section{General case}\label{general}

In this section, we provide algorithms and impossibility results for FMSM with general submodular objectives.

\subsection{Algorithms with lower bounds violation} %
\looseness=-1 We start by providing constant-factor approximation algorithms for FMSM that violate the fairness lower bounds by a constant factor. 
Such an algorithm can be directly obtained in the monotone case from the two-pass streaming algorithm of \citet[Section 4]{el2023fairness}. 
In the first pass, their algorithm finds a feasible set $S$; in the second pass, $S$ is extended to an $\alpha/2$-approximate solution which violates the lower bounds by a factor $2$, using an $\alpha$-approximation algorithm $\cA$ for monotone submodular maximization over the intersection of two matroids.
Using the state-of-the-art \emph{offline} $1/(2 + \epsilon)$-approximation algorithm of \citet[Theorem 3.1]{Lee2010} as $\cA$, instead of a streaming algorithm,
directly yields an improved $1/(4+\epsilon)$-approximation ratio.

We adapt this algorithm to non-monotone objectives by randomly dropping an appropriate number of elements from $S$ to balance the loss in objective value and the violation of the fairness lower bounds. 
We defer the details 
to \cref{sec:approx-constant-violation-appendix}.
\begin{restatable}{theorem}{constantviolation}\label{approx-constant-violation}
There exists a polynomial-time algorithm for non-monotone FMSM, which outputs a set $S$ such that $(i)$ $S \in \cI$, $(ii)$ $\floor{\beta\ell_c} \leq |V_c \cap  S| \le u_c$ for any color $c \in [C]$, and $(iii)$ $\E[f(S)] \ge  (1 - \beta) \mathrm{OPT}/(8 + \epsilon)$ for any $\beta \in [0, 1/2]$ and $\epsilon > 0$. 
\end{restatable}
\begin{proof}[Proof sketch]
We extend the result of \citet[Lemma 4.3 and 4.4]{el2023fairness} to non-monotone objectives; showing that dropping elements from $S$ allows us to only lose a factor $(1 - \beta)$, in expectation, in the objective value by \cref{bounded-sampling}, while keeping at least $\floor{\beta\ell_c}$ elements from each color.
Plugging in the state-of-the-art $1/(4 + \epsilon)$-approximation algorithm 
of \citet[Theorem 4.1]{Lee2010} as $\cA$ directly yields a $(1 - \beta)/(8 + \epsilon)$-approximation, in expectation.
\end{proof}
For example, if we accept a factor $2$ violation of the fairness lower bounds ($\beta = 1/2$), \cref{approx-constant-violation} gives a $1/(16 + \epsilon)$-approximation for non-monotone FMSM.

Note that if we use a streaming algorithm for $\cA$, the adapted algorithm becomes a two-pass streaming algorithm, as in the monotone case. In particular, using the state-of-the-art $1/7.464$-approximation algorithm by \citet[Theorem 19]{GargJS21} yields a $(1 - \beta) / 14.928$-approximation in the streaming setting (\cref{non-monotone-streaming}). 

\subsection{Hardness of rounding}\label{sec:rounding-hardness}
The optimal approximation algorithms for maximizing a submodular function over a matroid constraint in \cite{calinescu2011maximizing, Feldman2011} rely on first approximately solving a continuous relaxation of the problem obtained by replacing $f$ by its \emph{multilinear extension} and the matroid constraint by its convex hull, then rounding the solution, which can be done without any loss of utility using \emph{pipage or swap rounding} \cite{calinescu2011maximizing, Chekuri2010}. 
It is then natural to attempt a similar relax-and-round approach here:
first approximately solve the continuous relaxation of FMSM,
\begin{equation}\label{cont-relaxation}
    \max_{x \in P_\cF} F(x),
\end{equation}
and then round the obtained solution to a set in $\cF$.
Unfortunately, this approach fails in the presence of the fairness constraint! In particular, we show that the integrality gap of Problem \eqref{cont-relaxation}
is $\Omega(\sqrt{n})$, even for a monotone objective and a partition matroid. Hence, it is not possible to obtain better than $O(1/\sqrt{n})$-approximation using this approach.
\begin{restatable}{theorem}{integralitygap}\label{integrality-gap}
There is a family of FMSM instances where $f$ is monotone and $\cI$ is a partition matroid, for which the integral optimum solution has value $1$, but the multilinear extension admits a fractional solution of value $\Omega(\sqrt{n})$. 
\end{restatable}
\begin{proof}[Proof sketch]
We encode bipartite perfect matching as an instance of $\cF$ where $\cI$ is a partition matroid. Then we construct an instance where the objective is monotone and the bipartite graph has $\Theta(\sqrt{n})$ perfect matchings each of value $1$, while taking $x \in P_\cF$ as the average of the indicator vectors of the perfect matchings yields $F(x) = \Theta(\sqrt{n})$.
\end{proof}

\subsection{Algorithms with expected fairness}\label{sec:expected-fairness}

Motivated by \cref{integrality-gap}, we adapt the relax-and-round approach to ignore the fairness constraints during the rounding phase. 
This allows us to use the randomized swap rounding algorithm of \citet[Section IV]{Chekuri2010} 
to obtain an independent-set solution which preserves the same utility as the fractional solution, 
and which satisfies the fairness constraints \emph{in expectation}. This solution is further guaranteed, with constant probability, to 
 violate the fairness bounds by at most a multiplicative factor, expected to be small in practice. 
 This approach is inspired from \citet[Theorem 14]{CelisHV18}, who considered a special case of FMSM with a monotone objective and a uniform matroid, but with possibly non-disjoint color groups. \Cref{approx-violated-fairness} extends their result to any matroid and non-monotone objectives. 
 The non-monotone case is more challenging than the monotone case;  we show in \cref{sec:non-monotone-hardness} that it is impossible to obtain an approximation better than $(1 -   \min_{x \in  P_\cF  \cup (\1 - P_\cF)} \| x \|_\infty)$ for non-monotone FMSM in sub-exponential time, even for solving the continuous Problem \eqref{cont-relaxation}. 
 
\begin{restatable}{theorem}{expectedfairness}\label{approx-violated-fairness}
There exists a polynomial-time algorithm for FMSM which outputs a solution $S \in \cI$ such that  $\E[|S \cap V_c|] \in [\ell_c, u_c]$ and $\E[ f(S)] \geq \alpha \OPT$, 
where $\alpha = 1 - 1/e$ if $f$ is monotone and $\alpha = (1 -  \min_{x \in  P_\cF  \cup (\1 - P_\cF)} \| x \|_\infty - \epsilon)/4$ otherwise, for any $\epsilon>0$.
Moreover, the solution satisfies with constant probability the following for all $c \in [C],$
\begin{equation*} \label{eq:approx-fairness}
\left(1 -  \sqrt{\tfrac{3 \ln(2 C)}{\ell_c}}\right) \ell_c \leq |S \cap V_c| \leq \left(1 + \sqrt{ \tfrac{3 \ln(2 C)}{u_c}}\right) u_c.
\end{equation*}
\end{restatable}
\mtodo{Decided to remove the bound on the monotone objective with constant probability because that requires using $\ln(2 C/ (1 - e^{-\delta^2/16})$ in the fairness bounds instead of $\ln(2 C)$ which can be very large for small $\delta$. Note that the result in  \citep[Theorem 14]{CelisHV18} has the same issue, but there they can consider $C = \Omega(1)$, since they have a separete result for constant $C$ case, so they can make the probability of fairness bounds holding as close to $1$ as needed -- Added a remark about this in appendix.}
\begin{proof}[Proof] 
\looseness=-1
We first obtain an $\alpha$-approximation fractional solution $x \in P_\cF$ for $\max_{x \in P_\cF} F(x)$. For monotone objectives, we can use the continuous greedy algorithm of \citet[Section 3.1 and Appendix A]{Calinescu2011}, which achieves a $(1 - 1/e)$-approximation for maximizing the multilinear extension over any integral polytope, with high probability.\footnote{The result in \citep{Calinescu2011} is given for matroid polytopes, but it applies more generally for any integral polytope. \label{rmk-contgreedy}} Note that $P_\cF$ is integral as it is the convex hull of integral points.

For non-monotone objectives, we can use the Frank-Wolfe type algorithm of \citet[Sections 3.5 and 4.5]{du2022lyapunov} which achieves a $(1 - \min_{x \in  P} \| x \|_\infty - \epsilon)/4$-approximation for maximizing the multilinear extension over any polytope $P$, with high probability --- see also \citep[Section 3]{Mualem2023} for an explicit variant of the algorithm and its analysis. 
To obtain the approximation ratio $(1 -  \min_{x \in P_\cF \cup \1 - P_\cF}  \| x\|_\infty - \epsilon)/4$, we apply the algorithm of \citet{du2022lyapunov} to both Problem \eqref{cont-relaxation} and its complement: $\max_{x \in \1 - P_\cF} \bar{F}(x)$, where $\bar{F}$ is the multilinear extension of the complement $\bar{f}$ of $f$; $\bar{f}(S) = f(V \setminus S)$. 
Both the continuous greedy and the Frank-Wolfe type algorithms run in polynomial time if the constraint polytope is solvable,
which is the case for $P_\cF$ and $\1 - P_\cF$ by \cref{polytopes-properties}-\ref{linear-opt}.

Next we round the fractional solution $x$ to an independent (but not necessarily feasible) set $S \in \cI$ using the randomized swap rounding scheme of \citet[Section IV]{Chekuri2010}. The rounded solution is guaranteed to satisfy $\E[|S \cap V_c|] = x(V_c) \in [\ell_c, u_c]$,  $ \E[ f(S)] \geq F(x) \geq \alpha \OPT$ (see Theorem II.1 therein). The rest of the claim follows from the concentration bounds given in \citep[Theorem II.1]{Chekuri2010} and union bound.
See details in \cref{sec:expected-fairness-proofs}.
\end{proof}

As discussed in \cite{CelisHV18} -- see discussion under Theorem 14 therein -- the violation of the fairness constraints in the above theorem is expected to be small, i.e., $\sqrt{3 \ln(2 C) / u_c} \leq \sqrt{3 \ln(2 C) / \ell_c} \ll 1$, in typical fairness applications.  
However, in general, this violation 
can be arbitrarily bad. In particular, when $\cF$ encodes a bipartite perfect matching constraint, as in the proof of \cref{integrality-gap}, we have $C = \Theta(\sqrt{n})$ and $u_c  = \ell_c = 1$, hence $\sqrt{3 \ln(2 C) / u_c} = \sqrt{3 \ln(2 C) / \ell_c} \gg 1$.

Note that the approximation ratio for non-monotone objectives in \cref{approx-violated-fairness} can be as good as $1/4$ in some cases (for example if $\ell_c = 0$ for all $c$ or $u_c = |V_c|$ for all $c$), but it can also be arbitrarily bad. To illustrate this,  
we compute below the minimum $\ell_\infty$-norm of $P_\cF$ and its complement for two special cases of $P_\cF$. See \cref{sec:expected-fairness-proofs} for the full proofs. \\


\begin{restatable}[Uniform matroid]{example}{unifmatroidex}\label{packing-nbr-unif} Let $\cI$ be the uniform matroid with rank $k$. We order the color groups such that $\tfrac{\ell_1}{|V_1|} \leq \cdots \leq \tfrac{\ell_C}{|V_C|}$ and let $t$ be the largest index $t \in [C]$ such that $\tfrac{\ell_{t}}{|V_{t}|} \sum_{c = 1}^{t} |V_c| + \sum_{c= t+1}^C \ell_c \leq k$ ($t \ge 1$ is well-defined since $\cF \not = \emptyset$). Then  $\min_{x \in  P_\cF  \cup (\1 - P_\cF)} \| x \|_\infty= \min\{  \max_c \tfrac{\ell_c}{|V_c|}, 1 -  \min\{\tau, \min_{c \leq t} \tfrac{u_c}{|V_c|}\}\}$, where $\tau =  \tfrac{k - \sum_{c= t+1}^C \ell_c}{\sum_{c = 1}^{t} |V_c|}$.
\end{restatable}
\begin{proof}[Proof sketch]
We construct closed-form solutions to $\min_{x \in  P_\cF} \| x \|_\infty$ and $\min_{x \in \1 - P_\cF} \| x\|_\infty$ with $\ell_\infty$-norms $\max_c \tfrac{\ell_c}{|V_c|}$ and $1 - \min\{\tau, \min_{c \leq t} \tfrac{u_c}{|V_c|}\}$, respectively.
\end{proof}

The fairness bounds are often set such that the representation of color groups in a fair set is proportional to their representation in the ground set, i.e., $\ell_c, u_c \propto  k |V_c|/n$. If in particular $\ell_c \leq k |V_c|/n$ for all $c$, we get in \cref{packing-nbr-unif} $t = C, \tau = k / n$ and $\min_{x \in  P_\cF  \cup (\1 - P_\cF)} \| x \|_\infty \leq \min\{k/n, 1 - k/n\} \leq 1/2$. The corresponding approximation ratio in \cref{approx-violated-fairness} for non-monotone objectives is then at least $1/8 - \epsilon$.\\


\begin{restatable}[Bipartite perfect matching]{example}{matchingex}\label{packing-nbr-matching} 
Let $\cF$ be the feasible set 
arising from the bipartite perfect matching problem
in the proof of the integrality gap \cref{integrality-gap}.
We have $ \min_{x \in  P_\cF  \cup (\1 - P_\cF)} \| x \|_\infty = 1 - \Theta(1/\sqrt{n})$.
\end{restatable}
\begin{proof}[Proof sketch]
We show that the point $x$ used in the proof of \cref{integrality-gap} is a solution of $\min_{x \in  P_\cF  \cup (\1 - P_\cF)} \| x \|_\infty$ with $\|x\|_\infty = \|\1 - x \|_\infty = 1-\Theta(1/\sqrt{n})$.
\end{proof}

The corresponding approximation factor for \cref{packing-nbr-matching} in \cref{approx-violated-fairness} for non-monotone objectives is then $O(1/\sqrt{n})$. So relaxing the fairness constraints was not helpful in this case!

Recall that the $(1 - 1/e)$-approximation for monotone objectives in \cref{approx-violated-fairness} is tight even without the fairness constraint \cite{Feige98}. For non-monotone objectives, 
the $1/4$ factor is also tight for maximizing the multilinear extension of a submodular function 
over a  general polytope \citep[Theorem 5.1]{Mualem2023}; so improving it, if possible, would require more specialized algorithms for solving Problem \eqref{cont-relaxation}. %

\mtodo{figure out if the hardness result holds if fairness constraint is only satisfied in expectation, with $C=1$.}

\subsection{Uniform matroid case}\label{sec:uniform-matroid}


In the special case of the uniform matroid, both the fairness violation and the approximation ratio for non-monotone objectives given in \cref{approx-violated-fairness} can be improved.  
In particular, 
recall that a $(1 - 1/e)$-approximation can be obtained in this case for monotone objectives, without any violation of the fairness constraint \citep[Theorem 18]{CelisHV18}. For non-monotone objectives, we show that a $0.401 ( 1 - \min_{x \in  P_\cF  \cup (\1 - P_\cF)} \| x \|_\infty)$-approximation can be obtained in expectation, also without any violation of the fairness constraint.\footnote{Note that the $0.401$ factor in the approximation given in \cref{approx-unif} can be improved 
if a better approximation for non-monotone submodular maximization over a general or a partition matroid is developed in the future.}

\begin{restatable}{theorem}{approxunif}\label{approx-unif}
 There exists a polynomial-time algorithm for non-monotone FMSM where $\cI$ is a uniform matroid, which achieves 
$0.401 (1 - \min_{x \in  P_\cF  \cup (\1 - P_\cF)} \| x \|_\infty)$-approximation in expectation.
\end{restatable}
\begin{proof}[Proof sketch]
\citet[Theorem 5.2]{HalabiMNTT20} presented a polynomial-time algorithm with an expected $(1 - \max_c \tfrac{\ell_c}{|V_c|}) \alpha$-approximation for non-monotone FMSM over a uniform matroid, given any  $\alpha$-approximation algorithm for non-monotone submodular maximization over a matroid constraint.\footnote{The result in \cite{HalabiMNTT20} is given for the streaming setting, but it also holds offline. \label{rmk:nonmonotone-unif}}
Plugging in the state-of-the-art $0.401$-approximation of \citet[Corollary 1.3]{Buchbinder2023} then yields a $0.401 (1 - \max_c \tfrac{\ell_c}{|V_c|}) $-approximation in expectation. 
Recall from \cref{packing-nbr-unif} that $\min_{x \in  P_\cF } \| x \|_\infty =  \max_c \tfrac{\ell_c}{|V_c|}$.

To obtain the potentially better approximation factor $0.401 (1 - \min_{x \in \1 - P_\cF} \|x\|_\infty)$, we solve (as in \cref{approx-violated-fairness}) the complement problem $\max_{V \setminus S \in \cF} \bar{f}(S)$. 

We follow a similar strategy as \citet{HalabiMNTT20}. Namely, we first drop the lower bounds from the constraint. The resulting problem  $\max_{S \subseteq V} \{ \bar{f}(S) : |S \cap V_c| \leq |V_c| - \ell_c, ~ \forall c \in [C]\}$ is a non-monotone submodular maximization problem over a partition matroid. Hence, a solution $S$ with an expected $0.401$-approximation can be obtained for it.
Next, we augment $S$ to a feasible solution by carefully sampling enough random elements from each color to satisfy the lower bounds. Each element $e \in V \setminus S$ is added with probability at most $\bar{x}_e$, where $\bar{x}$ is a solution of $\min_{x \in \1 - P_\cF} \|x\|_\infty$ as in \cref{packing-nbr-unif}. Finally, we use \cref{bounded-sampling} to bound the loss in value resulting from the additional elements.
The sampling step is more involved here than in the original problem because $\bar{x}$ is non-integral.  See details in \cref{sec:uniform-matroid-proofs}. 
\end{proof}

\Cref{approx-unif} recovers the result of \citet{yuan2023group} for the special case where $\tfrac{\ell_c}{|V_c|} = a$ and $\tfrac{u_c}{|V_c|} = b$ for all $c \in [C]$, for some constants $a, b \in [0,1]$.\textsuperscript{\ref{remark-Yuan}}  Indeed, in that case $t = C$ and $\tau = k/n \geq a$ in \cref{packing-nbr-unif}, and $\min_{x \in  P_\cF  \cup (\1 - P_\cF)} \| x \|_\infty = \min\{a, 1 - \min\{k/n, b \}\}$, so the resulting approximation ratio is at least 
$0.401 \max\{1-  a, a\} \geq 0.2005$.

More generally, 
for the typical choice of fairness bounds discussed earlier in \cref{sec:expected-fairness}, where $\ell_c \leq k |V_c|/n $ for all $c$, the resulting approximation ratio is also at least $0.401/2 = 0.2005$.


\subsection{Non-monotone hardness}\label{sec:non-monotone-hardness}
In this section, we show that no approximation ratio better than $(1 -   \min_{x \in  P_\cF  \cup (\1 - P_\cF)} \| x \|_\infty)$ can be obtained for non-monotone FMSM in sub-exponential time, even without the matroid constraint.
\begin{restatable}{theorem}{nonmonotonehard}\label{non-monotone-hardness}
For any $r$ in the form $r = 1 - \tfrac{1}{t}$ with $t\geq 2$ and any $\epsilon>0$, a $(1 - r + \epsilon) $-approximation for FMSM that works for instances where $\min_{x \in  P_\cF  \cup (\1 - P_\cF)} \| x \|_\infty\leq r$ requires an exponential number of value queries.
This is also the case for the continuous relaxation of FMSM, $\max_{x \in P_\cF} F(x)$. Furthermore, this is true even in the two special cases where $(i)$ $C=2$ and $\cI = 2^V$ (no matroid constraint), $(ii)$ $C = 1$ and $\cI$ is a partition matroid. 
\end{restatable}
\begin{proof}[Proof sketch]
The proof follows from the hardness results given in \citep[Theorem 1.2 and  1.9]{Vondrak2013} for non-monotone submodular maximization over the \emph{bases} of a matroid, by showing that the hard instance used therein is an instance of both special cases $(i)$ and $(ii)$ of FMSM. See \cref{sec:non-monotone-hardness-proofs}.
\end{proof}

\Cref{non-monotone-hardness} implies that any constant-factor approximation for FMSM or its continuous relaxation requires an exponential number of value queries in general. 

Note that a uniform matroid with $k = n$ is equivalent to having no matroid constraint. 
In that case, \cref{packing-nbr-unif} shows that $\min_{x \in  P_\cF  \cup (\1 - P_\cF)} \| x \|_\infty= \min\{  \max_c \tfrac{\ell_c}{|V_c|}, 1 - \min_{c} \tfrac{u_c}{|V_c|}\}$, since $t = C$ and  $\tau = 1$.
Hence, \cref{packing-nbr-unif} and \cref{non-monotone-hardness}  yield a lower bound of $\max\{ 1 - \max_c \tfrac{\ell_c}{|V_c|}, \min_c \tfrac{u_c}{|V_c|}\}$ for the fairness constraint alone. This matches the lower bound given in \citep[Theorem 5.1]{HalabiMNTT20} for FMSM in the streaming setting -- the bound therein is $1 - \max_c \tfrac{\ell_c}{|V_c|}$ and is stated for the uniform matroid, but the hard instance used in the proof does not use any matroid and has $\min_c \tfrac{u_c}{|V_c|} = 0$.

On the other hand, note that \cref{packing-nbr-matching} does not necessarily imply a hardness of $O(1/\sqrt{n})$ for the bipartite perfect matching constraint, since the hard instance in \cref{non-monotone-hardness} is not an instance of that.

\mtodo{Can we show a hardness of $O(1/\sqrt{n})$  then for bipartite perfect matching for non-monotone FMSM?}

\section{Decomposable case}\label{sec:decomposable}

In this section, we study a special case of FMSM where $f$ is a decomposable submodular function over $\cF$ (\cref{decomposable-fcts}). We follow the relax-and-round approach discussed in \cref{sec:rounding-hardness}. Unlike the general case, this approach works here; we show that the integrality gap of Problem \eqref{cont-relaxation} is $1$ in this case. To that end, we show that the randomized swap rounding algorithm proposed in \citep[Section V]{Chekuri2010} for the intersection of two matroids can be applied to $\cF$.

We achieve this by reducing rounding on $\cF$ to rounding on the intersection of two matroids, one with the same equivalence classes (\cref{equiv-classes}) as $\cI$ and the other with the color groups $V_c$ as equivalence classes. 
Our reduction uses the following simple observation.
\begin{fact}\label{maximal-sets}Given any family of sets $\mathcal{J} \subseteq 2^V$, 
let $\tilde{\mathcal{J}}$ denote the family of \emph{extendable sets}: $\tilde{\mathcal{J}} = \{S \subseteq V \mid \text{ there exists } S' \in \mathcal{J} \text{ such that } S \subseteq S'\}$. Then $\mathcal{J}$ and $\tilde{\mathcal{J}}$ have the same maximal sets.
\end{fact}

When the matroid in $\cF$ is the uniform matroid, $\tilde{\mathcal{F}}$ was shown to form a matroid \citep[Lemma 4.1]{HalabiMNTT20}. We can then view any fair set as a base in a matroid, as shown by the following lemma. See \cref{sec:decomposable-proofs} for the proofs of this section. 

\begin{restatable}{lemma}{unifmaximalsets}\label{unif-maximal-sets}
Let $\tilde{\cC}_k := \{S \subseteq V \mid  \text{ there exists } S' \in \cC, |S'| \leq k \text{ such that } S \subseteq S' \}$. Then any fair set $S \in \cC$ is a base of the matroid $\tilde{\cC}_k$ with $k = |S|$ and vice versa.
\end{restatable}

We are now ready to give our reduction.

\begin{restatable}{theorem}{swaprounding}\label{swap-rounding}
There exists a polynomial-time rounding scheme which, given $x \in P_\cF$, returns a feasible set $S \in \cF$ that for any decomposable submodular function $f$ over $\cF$ satisfies
$\E[f(S)] \geq F(x)$.
\end{restatable}
\begin{proof}[Proof sketch]
Since $P_\cF$ is solvable (\cref{polytopes-properties}-\ref{linear-opt}), we can write $x$ as a convex combination of sets in $\cF$ in strongly polynomial time \citep[Theorem 5.15]{Schrijver2003}; i.e.,  $x = \sum_t \alpha_t \1_{I_t}$ for some  $I_t \in \cF, \alpha_t \geq 0, \sum_{t} \alpha_t = 1$. \Cref{unif-maximal-sets} implies that every set $I_t$ is in the intersection of the matroids $\cI$ and $\tilde{\cC}_{k}$, where 
$k$ is the size of the largest set $I_t$. Then $x$ is a valid input for swap rounding. If all sets $I_t$ have the same size $k$, then the returned rounded solution $S \in \cI \cap \tilde{\cC}_{k}$ will also have size $k$,
and thus is a base of $\tilde{\cC}_{k}$.
 By \cref{unif-maximal-sets}, $S$ is then also a fair set.
The claim then follows from \citep[Theorem II.3]{Chekuri2010} by noting that the equivalence classes of $\tilde{\cC}_{k}$ are the color groups $V_c$, and thus any decomposable submodular function over $\cF$ is also decomposable over  $\cI \cap \tilde{\cC}_{k}$. 
If sets $I_t$ do not have the same size, we reduce to the equal-sizes case by adding dummy elements. 
\end{proof}

\Cref{swap-rounding} allows us to obtain the same approximation factors as in \cref{approx-violated-fairness} but without violating the fairness constraint.

\begin{corollary}\label{approx-decomposable}
There exists a polynomial-time algorithm for decomposable FMSM 
which achieves, in expectation, $(1 - 1/e)$-approximation if $f$ is monotone and $(1 -  \min_{x \in  P_\cF  \cup (\1 - P_\cF)} \| x \|_\infty - \epsilon)/4$-approximation otherwise, for any $\epsilon>0$.
\end{corollary}
\begin{proof}
This follows directly from the proof of \cref{approx-violated-fairness} and \cref{swap-rounding}. Note also that both the continuous greedy and the Frank Wolfe type algorithms used to solve Problem \eqref{cont-relaxation} return a fractional solution $x$ which is already a convex combination of sets in $\cF$. So this step can be skipped in the proof of \cref{swap-rounding}.
\end{proof}
Recall that the fair submodular welfare problem (c.f., \cref{SWP}) is a special case of monotone decomposable FMSM. \Cref{approx-decomposable} then gives a $(1 - 1/e)$-approximation for it.

In the special cases where $C=1$ or $\cI$ is the uniform matroid, \cref{approx-decomposable} applies to any submodular objective, since $V_1 = V$ in the first case and all elements are equivalent in the second case, i.e., $\cG = \{ V\}$ in \cref{equiv-classes}. Hence, \cref{approx-decomposable} recovers the result of \citet[Theorem 18]{CelisHV18} for monotone objectives. 
For non-monotone objectives, the approximation guarantee in \cref{approx-decomposable} can be of course improved to the one given in \cref{approx-unif}.

It follows also that the hardness result in \cref{non-monotone-hardness} applies to decomposable FMSM too. 
\begin{corollary}\label{non-monotone-hardness-decomposable}
For any $r$ of the form $r = 1 - \tfrac{1}{t}$ with $t\geq 2$ and any $\epsilon>0$, a $(1 - r + \epsilon)$-approximation for decomposable FMSM that works for instances where $\min_{x \in  P_\cF  \cup (\1 - P_\cF)} \| x \|_\infty\leq r$ requires an exponential number of value queries.
This is also the case for the continuous relaxation of FMSM, $\max_{x \in P_\cF} F(x)$. Furthermore, this is true even in the two special cases where $(i)$ $C=2$ and $\cI = 2^V$ (no matroid constraint), $(ii)$ $C = 1$ and $\cI$ is a partition matroid. 
\end{corollary}
\begin{proof}
This follows directly from the proof of \cref{non-monotone-hardness}, by noting that any submodular function is decomposable over $\cF$ in both special cases; in $(i)$ $\cI$ is a $n$-uniform matroid and in $(ii)$ $C=1$. 
\end{proof}

\section{Future directions}
Our work leaves the following
tantalizing question open:
\emph{is there a constant-factor approximation to monotone FMSM}
(without fairness violations)?
Even more fundamentally:
is there a constant-factor approximation to the problem of maximizing a monotone submodular function over the set of \emph{perfect} matchings in a bipartite graph?
On one hand,
given our integrality gap in \cref{integrality-gap} for the multilinear extension,
a positive answer would require novel algorithmic techniques.
On the other hand,
in a plenary SODA talk, \citet{VondrakTalk} stated the following as a fundamental question:
``Can we approximate every maximization problem with a
monotone submodular objective (up to constant factors)
if we can approximate it with a linear objective?''
As maximum-weight perfect matching is in \textsf{P},
a superconstant hardness of approximation for the submodular perfect matching problem would imply a negative answer to Vondrák's question.

\subsubsection*{Acknowledgements}

We thank Rico Zenklusen, Jakab Tardos, and Moran Feldman for helpful discussions.

\bibliography{biblio}
\bibliographystyle{plainnat}
\clearpage
\appendix

\section{Algorithms for non-monotone FMSM  with lower bounds violation}\label{sec:approx-constant-violation-appendix}

In this section, we present a constant-factor approximation algorithm for FMSM, where the function $f$ is non-monotone, which violates the fairness lower bounds by a constant factor. The algorithm is a simple adaptation of the two-pass streaming algorithm in \cite{el2023fairness}, with both a streaming and offline variant. 
The first pass of our algorithm is the same (\cref{alg:firstpass}). It collects a maximal independent set $I_c \in \cI$ for each color \emph{independently} and then constructs a feasible solution $S$ in $\cup_c I_c$.
\begin{algorithm}
\caption{\firstpass in \cite{el2023fairness} \label{alg:firstpass}}
    \begin{algorithmic}[1]
        \STATE $I_c \leftarrow \emptyset$ for all $c = 1,...,C$
        \FOR{each element $e$ on the stream}
            \STATE Let $c$ be the color of $e$
            \STATE \textbf{If} $I_c + e \in \cI $ \textbf{then} $I_c \leftarrow I_c + e$
        \ENDFOR
        \STATE $S \leftarrow$ a max-cardinality subset of $\bigcup_c I_c$ in $\cI \cap \cI_C$, where $\cI_C = \{S \subseteq V \mid |V_c \cap S| \le \ell_c \; \ \forall c \in [C]\}$ \label{step:matroid-intersection}
        \STATE \textbf{Return} $S$ 
    \end{algorithmic}
\end{algorithm}

\firstpass runs in polynomial time, uses $O(k \cdot C)$ memory, and is guaranteed to return a feasible solution \citep[Theorem 4.2]{el2023fairness}.\footnote{In the offline setting, we do not need to collect the sets $I_c$; we can directly compute a feasible set in $V$. These sets are needed to maintain a low-memory usage $O(k \cdot C)$ in the streaming setting.}
Step~\ref{step:matroid-intersection} consists in invoking any offline polynomial-time max-cardinality matroid intersection algorithm (note that $\cI_C$ is a partition matroid).

The second pass of our algorithm is very similar to \cite{el2023fairness} (\cref{alg:secondpass}). The feasible set~$S$ output by \firstpass is again divided into sets $S_1$ and $S_2$. But instead of assigning elements deterministically to $S_1$ or $S_2$ in such a way that each lower bound is violated by at most a factor two, we randomly assign, for each $c \in [C]$, $\floor {\beta \ell_c}$ elements from $S \cap V_c$  to $S_1$, and $\floor {\beta \ell_c}$ from $(S \setminus S_1) \cap V_c $ to $S_2$, where $\beta \in [0, 1/2]$ is a user-defined parameter which controls the trade-off between fairness and objective value. This is possible since $S$ is feasible so it has at least $\ell_c$ elements of each color $c$. Note that both $S_1$ and $S_2$ are independent in $\cI$ and satisfy $\floor {\beta \ell_c} = |S_i \cap V_c| \leq \ell_c \le u_c$.

The goal of the second pass is to extend $S_1$ and $S_2$ into high-value independent sets while still satisfying the fairness upper bounds and the relaxed lower bounds. To that end, similarly as \cite{el2023fairness}, we define the following matroids:

\begin{align}\label{eq:color_matroid_upper}
        \cI^C &= \{X \subseteq V \mid |X \cap V_c| \le u_c \; \ \forall c \in [C]\}\,, \nonumber \\
        \cI_1 &= \{X \subseteq V \mid X \cup S_1 \in \cI\}\,, \nonumber \\
        \cI_2 &= \{X \subseteq V \mid X \cup S_2 \in \cI\}.
\end{align}

Then we use any algorithm $\mathcal{A}$ that maximizes a \emph{non-monotone} submodular function over two matroid constraints, and run two copies of it in parallel, one with matroids $\cI^C, \cI_1$ and one with matroids $\cI^C, \cI_2$. Let $R_1, R_2$ be their respective outputs. We add to $R_i$ as many elements as necessary
from $S_i$ to satisfy the relaxed lower bounds, and return the better of the two resulting solutions.

  \begin{algorithm}
        \caption{\secondpass \label{alg:secondpass}}
        \begin{algorithmic}[1]
            \STATE \textbf{Input:} Set $S$ from \firstpass, routine $\cA$, and $\beta \in [0, 1/2]$
            \STATE $S_1 \gets \emptyset$, $S_2 \gets \emptyset$ 
            \FOR{each color $c$}
            \STATE Let $B^c_1, B^c_2$ be sets of $\floor {\beta \ell_c}$ random elements from $S \cap V_c$ and $(S \cap V_c) \setminus B^c_1$, respectively
            \STATE $S_1 \gets S_1 \cup B^c_1$
            \STATE $S_2 \gets S_2 \cup B^c_2$
            \ENDFOR
            \STATE Define matroids $\cI^C$, $\cI_1,\cI_2$ as in \cref{eq:color_matroid_upper}
            \STATE Run two copies of $\cA$, one for $\cI^C, \cI_1$ and one for $\cI^C, \cI_2$, and let $R_1$ and $R_2$ be their outputs
            \FOR{$i=1,2$} \label{l:postprocessing-start}
                \STATE $S'_i \leftarrow R_i$
                \FOR{$e$ in $S_i$}\label{line:fill}
                    \STATE Let $c$ be the color of $e$
                    \STATE \textbf{If} $|S_i'\cap V_c| < u_c$ \textbf{then} $S_i'\gets S_i'+e$
                \ENDFOR
            \ENDFOR \label{l:postprocessing-end}
            \STATE \textbf{Return} $S' = \arg \max (f(S'_1), f(S'_2))$
        \end{algorithmic}
    \end{algorithm}
    

\secondpass only differs from the second-pass algorithm in \citep[Algorithm 2]{el2023fairness} by the way elements in $S$ are assigned to $S_1$ and $S_2$, and the requirement that $\cA$ works for non-monotone objectives. 
Let us analyze its performance. 
\begin{lemma} \label{lem:violation-bound}
The output $S'$ of \secondpass satisfies $(i)$ $S' \in \cI$, $(ii)$ $|S' \cap V_c| \leq u_c$, and  $(iii)$ $|S' \cap V_c| \geq \floor{\beta\ell_c}$ for any color $c \in [C]$.
\end{lemma}
\begin{proof}
The proof follows in a similar way to \citep[Lemma 4.3]{el2023fairness}. For $i \in \{1, 2\}$, we have:
\begin{enumerate}
    \item[(i)] By the definition of $\cI_i$ in \cref{eq:color_matroid_upper} the set $R_i$ output by $\cA$ satisfies $R_i \cup S_i \in \cI$. Hence, $S'_i \in \cI$ by downward-closedness since $S'_i \setminus R_i \subseteq S_i$.
    \item[(ii)] Since $R_i \in \cI^C$ and the elements added to it in the for loop on Line \ref{line:fill} never violate the upper bounds, we have $S'_i \in \cI^C$.
    \item[(iii)] By construction $S_i$ has at least $\floor{\beta\ell_c}$ elements of each color $c$. And for any color $c$ such that $|S'_i \cap V_c| < u_c$, all the elements in $S_i \cap V_c$ are added to $S_i'$. Hence, $|S_i' \cap V_c| \geq |S_i \cap V_c| \geq \floor{\beta\ell_c}$. 
\end{enumerate}
\end{proof}

\begin{lemma}\label{lem:approx-bound}
Assume that $\mathcal{A}$ is an $\alpha$-approximation algorithm for non-monotone submodular maximization over the intersection of two matroids. Then the output $S'$ of \secondpass satisfies $\E[f(S)] \geq \frac{(1 - \beta)\alpha}{2}  \OPT$.
\end{lemma}
\begin{proof}
We first give a lower bound on $\max (f(R_1), f(R_2))$ in the same way as in \citep[Lemma 4.4]{el2023fairness}.
Let $O$ be an optimal solution of FMSM.
From Lemma $2.1$ in \cite{el2023fairness}, we know that $O$ can be partitioned into two sets $O_1, O_2$ such that $O_1 \cup S_1 \in \cI$ and $O_2 \cup S_2 \in \cI$. Therefore $O_i \in \cI^C \cap \cI_i$ for $i \in \{1, 2\}$ by definition of $\cI_i$ and downward-closedness of $\cI^C$.  
Since $\cA$ is an $\alpha$-approximation algorithm, we have that
\begin{align} \label{eq:op}
\max (f(R_1), f(R_2)) \geq & \frac{f(R_1) + f(R_2)}{2} \nonumber \\ \geq& \frac{\alpha}{2}(f(O_1)+f(O_2)) \nonumber \\
\geq & \frac{\alpha}{2}(f(O_1 \cup O_2) + f(O_1 \cap O_2)) \nonumber \\
\geq & \frac{\alpha}{2} f(O) =  \frac{\alpha}{2} \OPT.
\end{align}
It remains to compare $f(R_i)$ with $f(S'_i)$ for $i \in \{1, 2\}$. To that end we define two non-negative submodular functions $g_1, g_2$ as follows:
\[
g_i(W) = f(W \cup R_i) \quad \quad \text{ for all } W \subseteq S \text{ and } i \in \{1,2\}.
\]
For any $i \in \{1,2\}$, we show that the set $S'_i \setminus R_i$ of elements added to $S'_i$ in the for loop on Line \ref{line:fill} contains each element of $S$ with probability at most $\beta$. By construction of $S_1, S_2$, we have for any color $c$ and any $e \in S \cap V_c$,
\begin{align*}
    P(e \in (S'_1 \setminus R_1) \cap V_c) &\leq P(e  \in S_1 \cap V_c)
    = \frac{\floor{\beta\ell_c}}{|S \cap V_c|}
    \leq \frac{\floor{\beta\ell_c}}{\ell_c} \leq \beta
\end{align*}
and 
\begin{align*}
     P(e \in (S'_2 \setminus R_2) \cap V_c) &\leq P(e  \in S_2 \cap V_c)\\ 
     &= P(e \not \in S_1 \cap V_c) P(e \in S_2 \cap V_c \mid e \not \in S_1 \cap V_c)\\
     &= (1 - \frac{\floor{\beta\ell_c}}{|S \cap V_c|})\frac{\floor{\beta\ell_c}}{|S \cap V_c| - \floor{\beta\ell_c}}\\
     &= \frac{\floor{\beta\ell_c}}{|S \cap V_c|}\\
     &\leq \frac{\floor{\beta\ell_c}}{\ell_c} \leq \beta.
\end{align*}
We apply \cref{bounded-sampling} to submodular functions $g_1$ and $g_2$. We get that 
\[
\E[f(S'_i)] = \E[g_i(S'_i \setminus R_i)] \geq (1 - \beta)g_i(\emptyset) = (1 - \beta) f(R_i).
\]

Combining that with \cref{eq:op}, we get 
\begin{align*}
\E[f(S')] =& \E[\max(f(S'_1), f(S'_2))]
\\ \geq & \max(\E[f(S'_1)], \E[f(S'_2)])
\\ \geq & (1 - \beta) \max(f(R_1),f(R_2)) \\
\geq& \frac{(1 - \beta)\alpha}{2} \OPT.
\end{align*}
\end{proof}

In the offline setting, we can use the state-of-the-art $1/(4 + \epsilon)$-approximation algorithm
of \citet[Theorem 4.1]{Lee2010} as $\cA$. \Cref{approx-constant-violation} then follows from \cref{lem:violation-bound} and \cref{lem:approx-bound}.
\constantviolation*

In the streaming setting, we can use the state-of-the-art $1/7.464$-approximation algorithm of \citet[Theorem 19]{GargJS21} as $\cA$. \Cref{non-monotone-streaming} then also follows from \cref{lem:violation-bound} and \cref{lem:approx-bound}. 
\begin{theorem}\label{non-monotone-streaming}
There exists a polynomial-time two-pass streaming algorithm for non-monotone FMSM, which uses $O(k \cdot C)$ memory, and outputs a set $S$ such that $(i)$ $S \in \cI$, $(ii)$ $\floor{\beta\ell_c} \leq |V_c \cap  S| \le u_c$ for any color $c \in [C]$, and $(iii)$ $\E[f(S)] \ge  (1 - \beta) \mathrm{OPT}/14.928$ for any $\beta \in [0, 1/2]$. 
\end{theorem}

\section{Missing proofs}

In this section, we present the proofs that were omitted from the main text. We restate claims for convenience.

\subsection{Proofs of \cref{sec:rounding-hardness}}\label{sec:rounding-hardness-proofs}
\integralitygap*
\begin{proof}
We will encode bipartite perfect matching as an intersection of matroid and fairness constraints.
Given a bipartite graph $(A \cup B, E)$, we set up a partition matroid to enforce that every vertex in $A$ has at most one adjacent edge in the solution, and fairness constraints to enforce that every vertex in $B$ has exactly one adjacent edge in the solution (we assign color $b$ to every edge $(a,b)$).
We then have $k = |A| = |B| = C$.

To define an instance, we then need to describe a bipartite graph and a monotone submodular function defined on its edges.
For parameters $t, s \ge 1$ we define the graph as follows:
there are two special vertices $1$ and $2$;
and
for $i=1,...,t$ we add a path from $1$ to $2$ of length $2s+1$ to the graph (thus inserting $2s$ new vertices).
Note that this graph is bipartite.
Let $O_i$ be the set of the $s+1$ ``odd" edges on the $i$-th path.
We define the submodular objective function as $f(S) = \sum_{i=1}^t \1_{S \cap O_i \ne \emptyset}$.

Note that the graph has exactly $t$ perfect matchings.
Indeed, the choice of mate for vertex $1$ determines the entire perfect matching, which must then be the union of $O_i$ (for some $i$) and the even edges from all the other paths.
In particular, for each $i$ there is only one perfect matching that intersects $O_i$.
Therefore the integral optimum is $1$.

However, if we take $x$ to be the average of the indicator vectors of the $t$ matchings,
then we get $x \in P_\cF$ and
\[
F(x) = \sum_{i=1}^t \left(1 - \prod_{e \in O_i} (1-x_e) \right) = t \cdot \left( 1 - \left(1 - \frac{1}{t} \right)^{s+1} \right) \,.
\]
If we set $s=t$, then $F(x) = \Omega(t)$, which is $\Omega(\sqrt{n})$ as $n = |E| = \Theta(s \cdot t)$.
\end{proof}

\subsection{Proofs of \cref{sec:expected-fairness}}\label{sec:expected-fairness-proofs}

We start by giving the full proof of \cref{approx-violated-fairness}. 

\expectedfairness*
\begin{proof}
We first obtain an $\alpha$-approximation fractional solution $x \in P_\cF$ for $\max_{x \in P_\cF} F(x)$. For monotone objectives, we can use the continuous greedy algorithm of \citet[Section 3.1 and Appendix A]{Calinescu2011}, which achieves a $(1 - 1/e)$-approximation for maximizing the multilinear extension over any integral polytope, both in expectation and with high probability.\textsuperscript{\ref{rmk-contgreedy}}
Note that $P_\cF$ is integral as it is the convex hull of integral points.

For non-monotone objectives, we can use the Frank-Wolfe type algorithm of \citet[Sections 3.5 and 4.5]{du2022lyapunov} which achieves a $(1 - \min_{x \in  P} \| x \|_\infty - \epsilon)/4$-approximation for maximizing the multilinear extension over any polytope $P$, both in expectation and with high probability --- see also \citep[Section 3]{Mualem2023} for an explicit variant of the algorithm and its analysis. Applying this algorithm to Problem \eqref{cont-relaxation} then yields a $(1 - \min_{x \in  P_\cF} \| x \|_\infty - \epsilon)/4$-approximation. 

To obtain the potentially better approximation ratio $(1 -  \min_{x \in \1 - P_\cF}  \| x\|_\infty - \epsilon)/4$, we solve the complement of Problem \eqref{cont-relaxation}: $\max_{x \in \1 - P_\cF} \bar{F}(x)$, where $\bar{F}$ is the multilinear extension of the complement $\bar{f}$ of $f$; $\bar{f}(S) = f(V \setminus S)$. 
It is easy to verify that the complement of a submodular function is also submodular. So we can indeed  use the Frank-Wolfe type algorithm on the complement problem, which yields $(1 -  \min_{x \in  \1 - P_\cF} \|x\|_\infty - \epsilon)/4$ approximation for it. 
Given an $\alpha$-approximate solution $\bar{x}$ for the complement problem, $1 - \bar{x}$ is an $\alpha$-approximate solution to the original problem. To see this, we show that $\bar{F}(x) = F(\1 - x)$:

\begin{align*}
    \bar{F}(x) &= \sum_{S \subseteq V} \bar{f}(S) \prod_{i \in S} x_i \prod_{j \in V \setminus S} (1 - x_j) \\
    &= \sum_{S' \subseteq V} {f}(S') \prod_{i \in V \setminus S'} x_i \prod_{j \in S'} (1 - x_j) \\ 
    &= F(\1 - x).
    \vspace{-10pt}
\end{align*}

Hence,  $\max_{x \in P_\cF} F(x) = \max_{x \in P_\cF} \bar{F}(\1 - x) = \max_{x' \in \1 - P_\cF} \bar{F}(x')$. Running the Frank-Wolfe type algorithm on both the original problem and its complement and picking the better solution then yields a $(1 -  \min_{x \in P_\cF \cup \1 - P_\cF}  \| x\|_\infty - \epsilon)/4$-approximation to Problem \eqref{cont-relaxation}.

Both the continuous greedy and the Frank-Wolfe type algorithms run in polynomial time if the constraint polytope is solvable,
which is the case for $P_\cF$ and $\1 - P_\cF$ by \cref{polytopes-properties}-\ref{linear-opt}.

Next we round the fractional solution $x$ to an independent (but not necessarily feasible) set $S \in \cI$ using the randomized swap rounding scheme of \citet[Section IV]{Chekuri2010}. The rounded solution is guaranteed to satisfy (see Theorem II.1 therein) $\E[|S \cap V_c|] = x(V_c) \in [\ell_c, u_c]$,  $ \E[ f(S)] \geq F(x) \geq \alpha \OPT$, and 
for any $c \in [C]$, $\delta_1 \geq 0$, and $\delta_2 \in [0,1]$,
\begin{align*}
\prob{|S \cap V_c| \geq (1 + \delta_1) u_c} &\leq \left(\frac{e^{\delta_1}}{(1 + \delta_1)^{( 1+ \delta_1)}}\right)^{u_c}\\
&\leq e^{-0.38 u_c \delta_1^2} & \text{ by \citep[Lemma 15]{CelisHV18}}\\ \\
\text{ and } \prob{|S \cap V_c| \leq (1 - \delta_2) \ell_c} &\leq e^{-\ell_c \delta_2^2 /2}.
\end{align*}
Then, by choosing $\delta_1 = \sqrt{3\ln(2 C) / u_c}$ and $\delta_2 = \sqrt{3 \ln(2 C) / \ell_c}$,
we get by union bound that $S$ satisfies $|S \cap V_c| \in [(1 - \sqrt{3\ln(2 C)/{\ell_c}}) \ell_c, (1 + \sqrt{3\ln(2 C)/{u_c}}) u_c)]$ for all $c \in [C]$, with probability at least $1 - 1/(2C)^{1.14} - 1/(2C)^{1.5} \ge 1 - 1/2^{1.14} - 1/2^{1.5} >  0.19$.
\begin{remark}
If $f$ is monotone, we can further guarantee that for any $\delta>0$, the solution in \cref{approx-violated-fairness} satisfies with constant probability the following  for all $c \in [C],$
\begin{equation*}
\left(1 -  \sqrt{\tfrac{3 \ln(2 C / (1 - e^{-\delta^2/16}))}{\ell_c}}\right) \ell_c \leq |S \cap V_c| \leq \left(1 + \sqrt{ \tfrac{3 \ln(2 C / (1 - e^{-\delta^2/16}))}{u_c}}\right) u_c
\end{equation*}
and $f(S) \geq (1 - 1/e - \delta) \OPT$.
\end{remark}
\begin{proof}
If $f$ is monotone, we can apply the lower-tail
concentration bound in \citep[Theorem II.2]{Chekuri2010}. To that end, we assume without loss of generality that for any $i \in V$, there exists a feasible set $S' \in \cF$ containing $i$, otherwise we can simply ignore such element. Note that we can easily check if an element $i$ satisfies this assumption, by maximizing a linear function over $\cF$, which assigns zero weights to all elements except $i$ and a non-zero weight to $i$. This can be done efficiently by \cref{polytopes-properties}-\ref{linear-opt}.
Under this assumption, the marginal values of $f$ are in $[0, \OPT]$, since for any $i \in V$, we have $f(i \mid S) \leq f(i) \leq f(S') \leq \OPT$ by submodularity and monotonicity. 
Scaling $f$ by $1/\OPT$, we obtain 
\begin{align*}
    \prob{f(S) \leq (1 - \delta) F(x)} &\leq e^{- F(x) \delta^2/(8 OPT)} \\
    &\leq e^{-\delta^2/16},
\end{align*}
for any $\delta>0$, since $F(x) \geq (1 - 1/e)  \OPT$ with high probability.
Hence, $S$ is guaranteed to satisfy $f(S) \geq (1 - \delta) F(x)  \geq (1 - 1/e - \delta) \OPT$ with probability $1 - e^{-\delta^2/16}$.

Then we again use the concentration bounds in \citep[Theorem II.1]{Chekuri2010} but now choosing $\delta_1 = \sqrt{\tfrac{3 \ln({2 C}/{(1 - e^{-\delta^2/16})})}{ u_c}}$ and $\delta_2 = \sqrt{\tfrac{3 \ln({2 C}/{(1 - e^{-\delta^2/16})})}{ \ell_c}}$. We get by union bound that $S$ satisfies all the bounds in the claim with probability at least $(1 - e^{-\delta^2/16} ) (1 - 1/(2C)^{1.14} - 1/(2C)^{1.5}) > 0.19 (1 - e^{-\delta^2/16})$.
\end{proof}

\mtodo{Some omitted details to add if reviewers ask for more details about time complexity: Continuous greedy returns a solution $x$ such that $F(x) \geq (1 - 1/e) \OPT$ with probability at least $1 - 1/(2n^2)$ (this is the probability that the estimates of the marginals gains are ``good", see Lemma 3.2 in \cite{Calinescu2011}. Even the ``clean'' approximation in Appendix A is given with high probability not deterministically). This needs $O(k^2 n)$ iterations where $k$ is the cardinality of the vertices of $P_\cF$, which is bounded by the rank of the matroid. And in each iteration it uses $O(n \log n ~ T^2)$ samples to estimate the marginals, where $T$ is the number of iterations. So the total complexity is $\tilde{O}(k^4 n^3)$. However this can be improved to $O(n^{5/3}/\epsilon^3)$ \citep[Corollary 13]{Mokhtari2020} or $O(n^2/\epsilon^4)$ \citep[Theorem 1.3]{Badanidiyuru2014a} which both achieve $(1 - 1/e) \OPT - \epsilon$ using variants of continuous greedy.
FW needs $O(\tfrac{D^2 L}{\epsilon F(x^*)})$ iterations, where $D$ is the diameter of $P_\cF$ in $\ell_2$-norm so $D \leq \sqrt{n}$ and $L$ is the smoothness of the multilinear extension so $L \leq n \max_{i} (F(i) - F(i | V \setminus i))$ (see for example \citep[Lemma C.1]{Hassani2017} - the proof therein is for monotone but it can be easily adapted to non-monotone). So the complexity of the algorithm is $O(\tfrac{n^2}{\epsilon \OPT})$ as long as $\OPT$ is not exponentially small the algorithm will finish in polynomial time.
Both  \cite{du2022lyapunov} and  \cite{Mualem2023} assume the gradient is computed exactly, which requires exponential time for $F$, but we can estimate it using similar strategy as in continuous greedy, the same result should hold but now with high probability and higher complexity.
}
\end{proof}

Next, we prove the statements of \cref{packing-nbr-unif} and \cref{packing-nbr-matching}.

\unifmatroidex*
\begin{proof}
First we compute $\min_{x \in  P_\cF} \| x \|_\infty$.
For any $x \in P_{\cF}$, we have $\ell_c \leq \sum_{i \in V_c} x_i \leq \|x\|_\infty |V_c|$ for all $c \in [C]$, hence $\|x\|_\infty \geq \max_c \tfrac{\ell_c}{|V_c|}$. We show that this lower bound is achieved at a feasible vector $x \in P_\cF$.
Let $x \in [0,1]^n$ such that $x_i  = \frac{\ell_c}{|V_c|}$ for all $i \in V_c$ and $c \in [C]$, then $\|x\|_\infty = \max_c \tfrac{\ell_c}{|V_c|}$ and $x \in P_\cF$ since $\sum_{i \in V_c} x_i = \ell_c \in [\ell_c, u_c]$ and $\sum_{i \in V} x_i = \sum_c \ell_c \leq k$ by the assumption that $\cF \not = \emptyset$. Hence $\min_{x \in  P_\cF} \| x \|_\infty = \max_c \tfrac{\ell_c}{|V_c|}$.

Next we compute $\min_{x \in \1 - P_\cF} \| x\|_\infty = \min_{x \in P_\cF} \|  \1 - x\|_\infty = 1 - \max_{x \in P_\cF} \min_{i \in V} x_i$. 
Let $x \in [0,1]^n$ such that 
$$x_i  =
\begin{cases}
 \tfrac{\ell_c}{|V_c|} &\text{for all $i \in V_c, c > t$}\\
\min\{\tau, \tfrac{u_c}{|V_c|}\} &\text{for all $i \in V_c, c \leq t$}.
\end{cases}
$$
We argue that $x \in P_\cF$. We have $x(V_c)  = \ell_c  \in [\ell_c, u_c]$ for all $c > t$ and  $x(V_c)  = \min\{ u_c, \tau |V_c| \}$ for all $c \leq t$. By definition of $t$ and $\tau$, we have $\tau \geq \tfrac{\ell_{t}}{|V_{t}|}$. Hence $x(V_c) \geq \min\{u_c, \tfrac{\ell_{t}}{|V_{t}|} |V_c| \} \geq \ell_c$ for all $c \leq t$. By definition of $\tau$, we also have $x(V) \leq \tau \sum_{c = 1}^{t} |V_c| + \sum_{c= t+1}^C \ell_c = k$.
We show next that $\min_{i \in V} x_i = \min\{\tau, \min_{c \leq t} \tfrac{u_c}{|V_c|}\}$. This holds trivially if $t = C$. Otherwise, by definition of $t < C$, we have  $$k \leq \tfrac{\ell_{t+1}}{|V_{t+1}|} \sum_{c = 1}^{t+1} |V_c| + \sum_{c= t+2}^C \ell_c = \tfrac{\ell_{t+1}}{|V_{t+1}|} \sum_{c = 1}^{t} |V_c| + \sum_{c= t+1}^C \ell_c, $$
hence $\tau \leq \tfrac{\ell_{t+1}}{|V_{t+1}|} = \min_{i \in \cup_{c > t} V_c} x_i$.
Thus $\min_{i \in V} x_i$ is attained at some color $c \le t$.
This proves that $\max_{x \in P_\cF} \min_{i \in V} x_i \geq \min\{\tau, \min_{c \leq t} \tfrac{u_c}{|V_c|}\}$.
To prove the upper bound, we observe that for any $x' \in P_\cF$, we have $\min_{i \in V} x'_i |V_c| \leq x'(V_c) \leq u_c$
for every $c \in [C]$
and $k \geq x' (V) \geq \sum_{c = t+1}^C \ell_c + \min_{i \in V} x'_i \sum_{c = 1}^{t} |V_c|$. Thus by definition of $\tau$, we get
$$\min_{i \in V} x'_i \leq \min\{\tau, \min_{c \in [C]} \tfrac{u_c}{|V_c|}\} \leq \min\{\tau, \min_{c \leq t} \tfrac{u_c}{|V_c|}\}.$$ 
Hence, $\min_{x \in \1 - P_\cF} \| x\|_\infty = 1 - \min\{\tau, \min_{c \leq t} \tfrac{u_c}{|V_c|}\}$ achieved at $\bar{x} = \1 - x$.
\end{proof}

\matchingex*
\begin{proof}
We prove that $\min_{x \in  P_\cF  \cup (\1 - P_\cF)} \| x \|_\infty= 1 - 1/t$ (recall that $t = \Theta(\sqrt{n})$ in \cref{integrality-gap}).
The point $x$ used in the proof of \cref{integrality-gap} puts $1/t$ value on the ``odd'' edges, and $1-1/t$ value on the ``even" edges (each of which belongs to all perfect matchings but one). Thus it has $\|x\|_\infty = \|\1 - x \|_\infty = 1-1/t$. This proves the $\le$ direction. For the $\ge$ direction, consider any $x \in P_\cF$. The point $x$ must satisfy $\sum_{e \in \delta(1)} x_e = 1$, where $\delta(1)$ is the set of edges incident on the special vertex $1$. As $|\delta(1)| = t$, we must have $x_e \le 1/t$ for some $e \in \delta(1)$. Then consider the path from vertex $1$ to $2$ containing $e$, and let $e'$ be the next edge after $e$ on that path. The common endpoint of $e$ and $e'$ has degree two, thus $x_{e'} = 1-x_e \ge 1-1/t$. Hence $\min\{ \|x\|_\infty, \|\1 - x\|_\infty\} \ge 1-1/t$.
\end{proof}

\subsection{Proofs of \cref{sec:uniform-matroid}} \label{sec:uniform-matroid-proofs}

In this section, we expand on some details that were omitted from the proof of \cref{approx-unif}.

\approxunif*
\begin{proof}

\citet[Theorem 5.2]{HalabiMNTT20} presented a polynomial-time algorithm with an expected $(1 - \max_c \tfrac{\ell_c}{|V_c|}) \alpha$-approximation for non-monotone FMSM over a uniform matroid, given any  $\alpha$-approximation algorithm for non-monotone submodular maximization over a matroid constraint.\textsuperscript{\ref{rmk:nonmonotone-unif}}
Plugging in the state-of-the-art $0.401$-approximation of \citet[Theorem 1.1]{Buchbinder2019} then yields a $0.401 (1 - \max_c \tfrac{\ell_c}{|V_c|}) $-approximation in expectation. 
Recall from \cref{packing-nbr-unif} that $\min_{x \in  P_\cF } \| x \|_\infty =  \max_c \tfrac{\ell_c}{|V_c|}$.

To obtain the potentially better approximation factor $0.401 (1 - \min_{x \in \1 - P_\cF} \|x\|_\infty)$, we solve (as in \cref{approx-violated-fairness}) the complement problem $\max_{V \setminus S \in \cF} \bar{f}(S)$, where $\bar{f}(S) = f(V \setminus S)$.  
We follow a similar strategy as \citet{HalabiMNTT20}. Namely, we first drop the lower bounds from the constraint. The resulting problem  $\max_{S \subseteq V} \{ \bar{f}(S) : |S \cap V_c| \leq |V_c| - \ell_c, ~ \forall c \in [C]\}$ is a non-monotone submodular maximization problem over a partition matroid. Hence, a solution $S$ with an expected $0.401$-approximation can be obtained for it.

Next, we augment the solution $S$ to a feasible one that we will denote $S^+$. 
Let $\bar{x}$ be a solution of  $\min_{x \in \1 - P_\cF} \|x\|_\infty$, which has a closed form as shown in \cref{packing-nbr-unif}.
For each color $c$, we sample a set $B_c \subseteq V_c$ of elements so that:
\begin{enumerate}
    \item $\prob{e \in B_c} = \bar{x}_e$ for all $e \in V_c$,
    \item $|B_c|$ is $\floor{\bar{x}(V_c)}$ or $\ceil{\bar{x}(V_c)}$ for all $c \in [C]$,
    \item $\sum_c |B_c|$ is $\floor{\bar{x}(V)}$ or $\ceil{\bar{x}(V)}$.
\end{enumerate}

This can be done as follows, inspired by \citep[Section 2.1]{Hartley1962}.
Assign consecutive disjoint intervals on the real line to elements, of length $\bar{x}_e$ for each $e \in V$, grouped by colors;
e.g., if $V_1 = \{e_1, e_2, ...\}$, the first two intervals would be $[0,\bar{x}_{e_1})$ and $[\bar{x}_{e_1},\bar{x}_{e_1}+\bar{x}_{e_2})$.
Now sample a random offset $\alpha \in [0,1)$, and define $B_c$ to be those elements of $V_c$ whose interval contains a point in $\mathbb{Z} + \alpha$.
Property 1 follows directly. For property 2, note that the union of intervals of all $e \in V_c$ is also an interval (as elements of the same color are grouped together), of length $\bar{x}(V_c)$; its intersection with $\mathbb{Z} + \alpha$ must thus be of size either $\floor{\bar{x}(V_c)}$ or $\ceil{\bar{x}(V_c)}$,
and every point in this intersection adds exactly one new element to $B_c$ since $\bar{x}_e \le 1$ for all $e$. Property 3 follows likewise.

Then, for every color such that $|S \cap V_c| < |B_c|$, add any $|B_c| - |S \cap V_c|$ elements from $B_c \setminus S$ to $S$.
This obtains $S^+$.
We argue that the augmented set $S^+$ is feasible
(for the complement problem $\max_{V \setminus S \in \cF} \bar{f}(S)$).
Note that $|B_c| \in \{\floor{\bar{x}(V_c)}, \ceil{\bar{x}(V_c)} \} \subseteq [|V_c| - u_c, |V_c| - \ell_c]$, since $\bar{x} \in \1 - P_\cF$ and the bounds are integral.
For the lower bounds, we have $|S^+ \cap V_c| \ge |B_c| \ge |V_c| - u_c$.
For the upper bounds, note that $|S^+ \cap V_c| = \max\{ |S \cap V_c|, |B_c| \} \leq  |V_c| - \ell_c$, since by definition of $S$, $|S \cap V_c| \le |V_c| - \ell_c$.
Moreover, $|S^+| \ge \sum_c |B_c| \ge \floor{\bar{x}(V)} \ge n-k$, since $\bar{x} \in \1 - P_\cF$ and $n-k$ is an integer.

To bound the loss in value resulting from the additional elements, we make use of \cref{bounded-sampling}.
Note that for any $e \in V_c$, $\prob{e \in S^+ \setminus S} \le \prob{e \in B_c} = \bar{x}_e \le \|\bar{x}\|_\infty$.
\Cref{bounded-sampling} and the definitions of $S$ and $\bar{x}$ then imply that $$\E[\bar{f}(S^+)] \geq (1 -  \|\bar{x} \|_\infty) \E[\bar{f}(S)] \geq 0.401 (1 - \min_{x \in \1 - P_\cF} \|x\|_\infty) \max_{V \setminus S \in \cF} \bar{f}(S).$$ 
Finally, taking the complement $V \setminus S^+$ yields a $0.401 (1 - \min_{x \in \1 - P_\cF} \|x\|_\infty)$-approximation to the original problem, in expectation.
\end{proof}

\subsection{Proofs of \cref{sec:non-monotone-hardness}}\label{sec:non-monotone-hardness-proofs}
 In this section, we give the full proof of \cref{non-monotone-hardness}.
\nonmonotonehard*
\begin{proof}
\citet[Theorem 1.2, 1.9 and Section 2]{Vondrak2013} showed that any algorithm achieving a better than $(1 - r)$-approximation for the problem of maximizing a non-monotone submodular function over the \emph{bases} of a matroid $\cI$, where $\min_{x \in P_\cI : \sum_i x_i = k} \|x\|_\infty \leq r$, requires exponentially many value queries. This holds  also for the continuous relaxation of the problem (see Theorem 1.9 and Section 2 therein). 

We show that the hard instance used in the proofs of the lower bounds in \cite{Vondrak2013} is a special case of FMSM with $  \min_{x \in  P_\cF  \cup (\1 - P_\cF)} \| x \|_\infty\leq r$.
In particular, the hard instance uses a ground set $V = A \cup B$, where $A$ and $B$ are two disjoint sets, each of size $t \times m$, for some large number $m$, and a partition matroid base constraint $\mathcal{B} = \{S : |S \cap A| = m \text{ and }|S \cap B|= (t-1) m\}$ (see Section 2, Theorems 1.8 and 1.9 therein). 

We can express this matroid base constraint via $\cF$ in two ways. First, any matroid base constraint is a special case of $\cF$ where $C=1$, $V = V_1$, and $\ell_1 = u_1 = k$ (the rank of the matroid), i.e., the fairness constraint reduces to $|S| = k$. In the case of~$\mathcal{B}$, we have $k = t \times m$ and $\min_{x \in  P_\cF  \cup (\1 - P_\cF)} \| x \|_\infty= \min_{x \in P_\cI : \sum_i x_i = k} \|x\|_\infty =  1 - \tfrac{1}{t}$. 
Alternatively, we can express $\mathcal{B}$ with a fairness constraint alone, with two color groups $V_1 = A$, $V_2 = B$, $\ell_1 = u_1 = m$, and $\ell_2 = u_2 = (t-1) m$. 
\end{proof}

\subsection{Proofs of \cref{sec:decomposable}}\label{sec:decomposable-proofs}

In this section, we prove \cref{unif-maximal-sets} and \cref{swap-rounding}.

\unifmaximalsets*
\begin{proof}
Note that any fair set is a maximal set of the $k$-truncation of $\cC$ given by $\cC_k:= \{S \in \cC \mid |S| \leq k\}$ with $k = |S|$. The family $\cC_k$ is a special case of $\cF$ where $\cI$ is the $k$-uniform matroid, so the corresponding family of extendable set $\tilde{\cC}_k$ is a matroid \citep[Lemma 4.1]{HalabiMNTT20}.
The claim then follows from \cref{maximal-sets}.
\end{proof}

\swaprounding*
\begin{proof}
The randomized rounding scheme of \citet[Section V]{Chekuri2010} takes as input  a fractional solution $x' \in P_{\cI_1} \cap P_{\cI_2}$, where $\cI_1$ and $\cI_2$ are two matroids, expressed as 
a convex combination of sets in $\cI_1\cap \cI_2$, 
and returns a set $S' \in \cI_1 \cap \cI_2$ such that $\E[f'(S')] \geq F'(x')$ for any submodular function $f'$ which decomposes over the equivalence classes of $\cI_1$ or $\cI_2$, 
i.e., $f'(S) = \sum_{G \in \cG_1} f'_{1, G}(S \cap G) + \sum_{G \in \cG_2} f'_{2, G}(S \cap G),$ where $f'_{1, G}, f'_{2, G}$ are submodular functions and $\cG_1, \cG_2$ are the respective families of equivalence classes of $\cI_1$ and $\cI_2$ (see Theorem II.3 therein). If the sets in the input convex combination are all of the same size $k$, then the returned set will have size $k$ too\footnote{This is not explicitly stated in \citep[Section V]{Chekuri2010}, but can be easily deduced from the description of the merging operation therein.}.

Given $x \in P_\cF$, since $P_\cF$ is solvable (\cref{polytopes-properties}-\ref{linear-opt}), we can write $x$ as a convex combination of sets in $\cF$ in strongly polynomial time \citep[Theorem 5.15]{Schrijver2003}; i.e.,  $x = \sum_t \alpha_t \1_{I_t}$ for some  $I_t \in \cF, \alpha_t \geq 0, \sum_{t} \alpha_t = 1$.

Let $k$ be the size of the largest set $I_t$. 
\Cref{unif-maximal-sets} implies that every set $I_t$ is in the intersection of the matroids $\cI$ and $\tilde{\cC}_{k}$, where $\tilde{\cC}_{k}$ is defined as in the lemma.
Then $x$ is already a valid input for swap rounding. If all sets $I_t$ have the same size, then the returned rounded solution $S \in \cI \cap \tilde{\cC}_{k}$ is a base of $\tilde{\cC}_{k}$, and thus also a fair set by \cref{unif-maximal-sets}. 
So $S$ is feasible. Otherwise, $S$ is not necessarily feasible. In that case, we first need to  
add  dummy elements to make all sets $I_t$ of equal size. 

Let $E$ be a set of $k$ new dummy elements (not in $V$). For every $t$, we add enough dummy elements from $E$ to $I_t$ to obtain a set $I'_t$ of size $k$. We also extend the matroid $\cI$ to $\cI^+ = \{ S \subseteq V \cup E \mid S \setminus E \in \cI\}$ and the collection of fair sets $\cC$ to $\cC^+= \{ S \subseteq V \cup E \mid S \setminus E \in \cC\}$. It is easy to verify that $\cI^+$ is also a matroid. Observe that $\cC^+$ is also a collection of fair sets over the extended ground set, since $\cC^+= \{ S \subseteq V \cup E \mid |S \cap V_c| \in [\ell_c, u_c] ~\forall c \in [C], |S \cap E| \in [0, |E|]\}$, so \cref{unif-maximal-sets} applies to it. Moreover, $\tilde{\cC}^+_{k} =\{S \subseteq V \cup E \mid  \text{ there exists } S' \in \cC^+, |S'| \leq k \text{ such that } S \subseteq S' \}$ is a matroid 
of rank $k$, since any set in $\tilde{\cC}^+_{k}$ has size at most $k$ and there exists $I_t \in \tilde{\cC}^+_{k}$ with size $k$ for some~$t$. 

Let $x' = \sum_t \alpha_t \1_{I'_t}$. Since every set $I'_t$ is in the intersection of the matroids $\cI^+$ and $\tilde{\cC}^+_{k}$, $x'$ is a valid input to swap rounding. And since all $I'_t$ have size~$k$, the returned rounded solution $S' \in \cI^+ \cap \tilde{\cC}^+_{k}$ will also have size $k$. Hence, $S'$ is a base of $\tilde{\cC}^+_{k}$, and thus $S' \in \cC^+$ by \cref{unif-maximal-sets}.
Removing the dummy elements from $S'$ then yields a feasible solution $S = S' \setminus E \in \cI \cap \cC$.

It remains to show that the solution $S$ preserves the value of $F(x)$ for any decomposable submodular function $f$ over $\cF$. Define $f':2^{V\cup E} \to \R_+$ as $f'(S) = f(S \setminus E)$, then its multilinear extension $F': [0,1]^{n+k} \to \R_+$  is given by $F'(x') = F(x'_V)$ where $x'_V \in [0,1]^n$ is the vector corresponding to the entries of $x'$ in $V$.
We observe that $f'$ decomposes over the equivalence classes of $\cI^+$ or $\tilde{\cC}^+_{k}$, with $f'_{1, G}(S) = f_{1, G}(S \setminus E)$ and $f'_{2, c}(S) = f_{2, c}(S \setminus E)$, where $f_{1, G}, f_{2, c}$ are the functions in the decomposition of $f$ (\cref{decomposable-fcts}). One can verify that $f', f'_{1, G}, f'_{2, c}$ are submodular since $f, f_{1, G}, f_{2, c}$ are submodular.
The solution $S$ returned by swap rounding then satisfies $\E[f'(S')] = \E[f(S)] \geq F'(x') = F(x)$.
\end{proof}

It is worth mentioning that \citet[Appendix C.2]{el2023fairness} have already shown a reduction of FMSM to submodular maximization over the intersection of two matroid bases. However, their result cannot be used to prove \cref{swap-rounding}, since the matroids in their reduction have different equivalence classes than the ones of $\cI$ and the color groups. Using their reduction in \cref{swap-rounding} would require considering a much more restricted class of decomposable submodular functions (ones which are completely separable over $V$, or over subgroups of colors obtained by splitting each color group into two).


\section{Relation to multi-objective submodular maximization}\label{sec:multiobj-appendix}

In this section, we discuss the connection of FMSM to multi-objective submodular maximization, where given $m$ submodular functions $f_i$, the goal is to maximize their minimum $\min_{i \in [m]} f_i(S)$.

 We can write FMSM as a multi-objective submodular maximization problem over two matroids:
 \begin{equation}\label{eq:FMSM-multiobj}
     \max_{S \subseteq V} \{ \min_{i \in [C+1]} f_i(S) : S \in \mathcal{I}, |S \cap V_c | \leq u_c \quad  \forall c \in [C] \},
 \end{equation}
where $f_{C+1} = f / \OPT$ and $f_c = |S \cap V_c| /\ell_c$ for all $c \in [C]$.
 Recall that $\OPT$ is the optimal value of FMSM, which can be guessed. 
 A solution $S$ of Problem \eqref{eq:FMSM-multiobj} is then also an optimal solution to FMSM.
To the best of our knowledge, there are no existing work on multi-objective submodular maximization over two matroids.

 In the special cases where there are no upper bounds ($u_c = |V_c|$) or $\mathcal{I}$ is the uniform matroid, and $f$ is monotone, one can apply the result of \citep[Theorem 7.2]{Chekuri2009} on multi-objective submodular maximization over a single matroid to obtain a $(1 - 1/e)$-approximation for monotone FMSM.
But the resulting runtime will be exponential in $\sum_c \ell_c$ as the algorithm requires guessing the $\ell_c$ elements belonging to the optimal solution for each group $c$, by brute force enumeration.
 Recall though that when $\mathcal{I}$ is the uniform matroid, a polynomial time $(1 - 1/e)$-approximation for monotone FMSM was provided in \citep[Theorem 18]{CelisHV18} under the assumption that the groups are disjoint, which is not assumed in \citep{Chekuri2010}. \citet{CelisHV18} also provided a $(1 - 1/e)$-approximation algorithm for the overlapping groups case, but its runtime is also exponential in $C$ (see Theorem 20 therein).
 
 Finally, if there are no upper bounds ($u_c = |V_c|$), $\mathcal{I}$ is the uniform matroid, and $f$ is monotone, then the constraint in Problem \eqref{eq:FMSM-multiobj} becomes a simple cardinality constraint. In this case, a $(1 - 1/e - o(1))$-approximation for monotone FMSM can be obtained
 in polynomial-time if $k = o(n)$, even with overlapping groups \cite{ Udwani2018,TsangWRTZ19}.
 
 The line of work \cite{Wang2024, TsangWRTZ19, tang2023beyond} which studied submodular maximization under a different notion of fairness can  all be formulated as a monotone multi-objective submodular maximization problem over a cardinality constraint.

\end{document}